\documentclass[runningheads]{llncs}
\usepackage[misc]{ifsym}
\usepackage{amsmath,amsfonts,amssymb,stmaryrd}
\usepackage{booktabs}
\usepackage{ dsfont }
\usepackage{graphicx}
\usepackage{multirow}
\usepackage[capitalize]{cleveref}
\usepackage{tabularx}
\usepackage{wrapfig}
\newcolumntype{Y}{>{\centering\arraybackslash}X}
\newtheorem{observation}{Observation}
\usepackage{color}

\def\Z{\mathbb{Z}}
\def\R{\mathbb{R}}

\newcommand{\Mod}[1]{\ \mathrm{mod}\ #1}

\def\our{ChiENN}

\begin{document}
%


\title{\our{}: Embracing Molecular Chirality with Graph Neural Networks}

\titlerunning{\our{}: Embracing Molecular Chirality with Graph Neural Networks}

\author{Piotr Gaiński\inst{1}(\Letter) \and
Michał Koziarski\inst{2, 3} \and
Jacek Tabor\inst{1} \and
Marek \'Smieja\inst{1}}

\tocauthor{Piotr~Gaiński et al.}
\toctitle{\our{}: Embracing Molecular Chirality with Graph Neural Networks}
%
%
\institute{Faculty of Mathematics and Computer Science, Jagiellonian University, \\Krak\'ow, Poland \\
\and
Mila - Quebec AI Institute, Montreal, Quebec, Canada \\ 
\and
Université de Montréal, Montreal, Quebec, Canada \\
\email{piotr.gainski@doctoral.uj.edu.pl} 
}
\maketitle              
\begin{abstract}
Graph Neural Networks (GNNs) play a fundamental role in many deep learning problems, in particular in cheminformatics. However, typical GNNs cannot capture the concept of chirality, which means they do not distinguish between the 3D graph of a chemical compound and its mirror image (enantiomer). The ability to distinguish between enantiomers is important especially in drug discovery because enantiomers can have very distinct biochemical properties. In this paper, we propose a theoretically justified message-passing scheme, which makes GNNs sensitive to the order of node neighbors. We apply that general concept in the context of molecular chirality to construct Chiral Edge Neural Network (ChiENN) layer which can be appended to any GNN model to enable chirality-awareness. Our experiments show that adding ChiENN layers to a GNN outperforms current state-of-the-art methods in chiral-sensitive molecular property prediction tasks.

\end{abstract}

\section{Introduction}

Recent advances in Graph Neural Networks (GNNs) have revolutionized cheminformatics and enabled learning the molecular representation directly from chemical structures \cite{first_gnn,gat}. GNNs are widely used in molecular property prediction \cite{dmpnn,mat,unimol,molclr}, synthesis prediction \cite{gnn_syn1,gnn_syn2}, molecule generation \cite{maziark2,guacamol,maziarz1}, or conformer generation \cite{gen,gen2,gen3,geomol}. Surprisingly, typical GNNs cannot capture the concept of chirality, roughly meaning they do not distinguish between a molecule and its mirror image, called enantiomer (see \cref{fig:chiral-molecule}). Although enantiomers share many physical, chemical, and biological properties, they may behave remarkably differently when interacting with other chiral molecules, e.g. chiral proteins. For this reason, capturing chirality is critical in the context of drug design \cite{chiral1,chiral2,chiral3,chiral4,chiral5} and should not be ignored by the design of GNN architecture. 

\begin{figure}
  \includegraphics[width=0.7\linewidth]{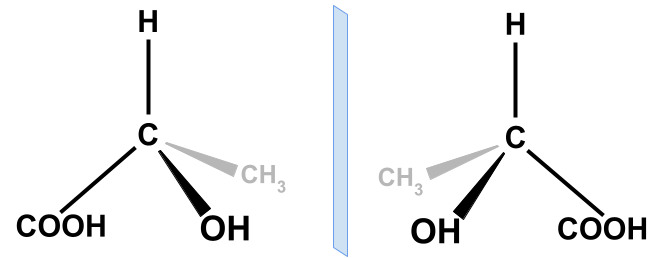}\centering
  \caption{An example of a chiral molecule (left) and its mirror image (right).}
  \label{fig:chiral-molecule}
\end{figure}

A chiral molecule is a molecule with at least one chiral center which is usually a carbon atom with four non-equivalent constituents. The mirror image of a chiral molecule, called an enantiomer, cannot be superposed back to the original molecule by any combination of rotations, translations, and conformational changes (see \cref{fig:chiral-molecule}). Therefore, enantiomers are molecules with different bond arrangements and the same graph connectivity. There are many examples of chiral drugs used in pharmacy whose enantiomers cause substantially different effects \cite{chiral1}. For instance (S)-penicillamine is an antiarthritic drug while its enantiomer (R)-penicillamine is extremely toxic \cite{wiki_chiral}. 

Actually, chirality can be a characteristic of any class of graphs embedded in euclidean space (where we have an intuitive notion of reflection). 
For instance, \cref{fig:road_map} shows two 2D road maps that are mirror images of each other and possess different properties. 
For this reason, modeling chirality in GNNs is not restricted to the chemical domain. 

\begin{wrapfigure}{r}{0.48\textwidth}
  \includegraphics[width=\linewidth]{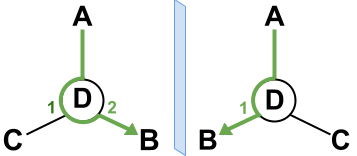}\centering
  \caption{An illustration of a road map (left) and its mirror image (right). We see that the maps share the same connectivity between cities, however, to get from city $A$ to city $B$ one has to take the second exit on a roundabout $D$ for the left map, and the first exit for the right map.}
  \label{fig:road_map}
\end{wrapfigure}

In this paper, we propose and theoretically justify a novel order-sensitive message-passing scheme, \linebreak
which makes GNNs sensitive to chirality. In contrast to existing methods of embracing chirality, our framework is not domain specific and does not rely on arbitrary chiral tagging or torsion angles (see \cref{sec:related_work}). The only inductive bias our method introduces to a GNN is the dependency on the orientation of the neighbors around a node, which lies at the core of chirality.

The key component of the proposed framework is the message aggregation function. In a typical GNN, the messages incoming to a node from its neighbors are treated as a set and aggregated with a permutation-invariant function (sum, max, etc.). It makes the model unable to distinguish between chiral graphs with the same connectivity, but with different spatial arrangements. We re-invent this approach and introduce a message aggregation function that is sensitive to the spatial arrangement (order) of the neighbors. Our approach can be used in any chiral-sensitive graph domain where chirality can be expressed by an order of the neighboring nodes.

We apply that general order-sensitive message-passing framework in the context of molecular chirality to construct Chiral Edge Neural Network (ChiENN) layer. The ChiENN layer can be appended to most molecular GNN models to enable chirality sensitiveness. Our experiments show that ChiENN can be successfully used within existing GNN models and as a standalone model consisting of stacked ChiENN layers. In both cases, ChiENN outperforms current state-of-the-art methods in chiral-sensitive molecular property prediction tasks by a large margin. We make our code publicly available\footnote[1]{\url{https://github.com/gmum/ChiENN}}. 

Our contributions are as follows:
\begin{enumerate}
    \item We propose and theoretically justify a general order-sensitive message-passing scheme. Our method can be adapted to any chiral-sensitive graph domain where chirality can be expressed by an order of the neighboring nodes (\cref{sec:osmp-scheme}).
    \item We use the proposed framework to construct a novel ChiENN layer that enables chirality awareness in any GNN model in the domain of molecular graphs  (\cref{sec:chienn}). The proposed ChiENN can be applied to any 3D graph task with the notion of chirality.
    \item We evaluate and analyze the ChiENN layer and show that it outperforms current state-of-the-art methods in chiral-sensitive molecular property prediction tasks (\cref{sec:experiments}).
\end{enumerate}

\section{Related Work}
\label{sec:related_work}
\subsubsection{Explicit Tagging of Chiral Center.}
The most common approach for incorporating chirality into GNN is to use local or global chiral tags \cite{tag1,tag2,tag3}. Both local and global tagging can be seen in the following way. Every carbon atom with four non-equivalent constituents, called a chiral center, is given a tag (CCW or CW) describing the orientation of its constituents. The orientation is defined using the enumeration of constituents computed by an arbitrary algorithm. The constituent with the highest number (4) is positioned so that it points away from the observer. The curve passing through the constituents with numbers 1, 2, and 3 respectively determines a clockwise (CW) or counterclockwise (CCW) orientation of the chiral center. Although enumeration algorithms for global and local tagging differ (the latter is not explicitly used in practice), the expressivity of both methods is limited, as we show in \cref{sec:experiments}.

\subsubsection{3D GNNs with torsion angles.}
Some recent GNN models enrich graphs with 3D information, like distances between atoms \cite{geometric,mat}, angles between bonds \cite{schnet,dimenet}, and torsion angles between two bonds joined by another bond \cite{spherenet,gemnet}. As distances and angles are invariant to chirality, the torsion angles (that are negated upon reflection) are required for 3D GNN to express the chirality. However, even access to a complete set of torsion angles does not guarantee expressivity in chiral-sensitive tasks as shown in \cite{chiro}. Torsion angles are sensitive to bond rotations and can also be negated by the reflection of a non-chiral molecule. In \cite{chiro}, the authors proposed the ChIRo model that instead of embedding single torsion angles, embeds sets of torsion angles with a common bond. ChIRo is the current state-of-the-art method for chiral-sensitive tasks. In contrast to ChIRo, our proposed method does not incorporate distances, angles, or torsion angles. It only relies on the orientation of neighbors around a node, making it more general and easily adaptable to other chiral graph domains. Moreover, our experiments show that the ChiENN layer outperforms ChIRo by a large margin on chiral-sensitive molecular tasks (see \cref{sec:experiments}).

\subsubsection{Changing Aggregation Scheme.}
The method most related to our approach is the Tetra-DMPNN model from \cite{tetra-dmpnn} which replaces a classic message-passing scheme with a chiral-sensitive one. The proposed aggregation scheme is guided by local chiral tags, meaning that it relies on some arbitrary rules for enumerating neighbors and cannot be applied to nodes other than chiral centers. Moreover, the Tetra-DMPNN method is computationally expensive and does not scale with the number of possible neighbors of a chiral center, making the model useful only in the context of chemistry. Our approach provides a general, efficient, and scalable chiral-sensitive message passing and outperforms the Tetra-DMPNN on chiral-sensitive molecular tasks by a large margin (see \cref{sec:experiments}).


\section{Order-Sensitive Message-Passing Scheme}
\label{sec:osmp-scheme}

\subsubsection{Setting.} Let us consider a directed graph $G=(X, E)$ in which every node $x_i \in X$ is represented by a $N$-dimensional encoding ($x_i \in \R^N$). Edge $e_{ij}$ connects nodes $x_i$ and $x_j$ and is represented by $M$-dimensional encoding ($e_{ij} \in \R^M$). 

In addition, we assume that for every node $x \in X$, we are given an order $o$ of all its neighbors $o=(x_0,x_1,\ldots,x_{d-1})$. The order of neighbors forms a sequence, which stands in contrast to typical graphs, where neighbors are treated as an unordered set. Given a permutation $\pi$ on $\{0,1,\ldots,d-1\}$, we assume that two orders $o_1=(x_0,\ldots,x_{d-1})$ and $o_2=(x_{\pi(0)},\ldots,x_{\pi(d-1)})$ are equivalent if and only if $\pi$ is a shift i.e. $\pi(i) = (i+k) \Mod d$, for a fixed $k \in \Z$. In other words, the neighbors form the sequence on a ring.

One of the most common mechanisms in GNN is message-passing, which updates the representation of a node $x$ by the information coming from its neighbors $(x_0,\ldots,x_{d-1})$, which can be written as:
$$
x'=f(x;x_0,\ldots,x_{d-1}).
$$

In this paper, we are going to describe the general message-passing scheme, which is aware of the neighbors' order. Before that, we discuss possible choices of the aggregation function $f$.

\subsubsection{Vanilla message-passing as a permutation-invariant transformation.} Let us first discuss a basic case, where $f$ is a permutation-invariant function, i.e.
$$
f(x;x_0,\ldots,x_{d-1})=f(x;x_{\pi(0)},\ldots,x_{\pi(d-1)}),
$$
for every permutation $\pi$ of $\{0,1,\ldots,d-1\}$. This aggregation ignores the order of neighbors and lies in a heart of typical GNNs.

Let us recall that $f$ is permutation-invariant with respect to $\{x_0,x_1,\ldots,x_{d-1}\}$ if and only if it can be decomposed in the form \cite{deepsets}:
$$
f(x_0,x_1,\ldots,x_{d-1}) = \rho(\sum_{i=0}^{d-1} \phi(x_i)),
$$
for suitable transformations $\phi$ and $\rho$. In the context of graphs, a general form of a permutation-invariant aggregation of neighbors  $\{x_0,x_1,\ldots,x_{d-1}\}$ of $x$ is:
\begin{equation}
    \label{eq:vanilla-message-passing}
    x'=f(x;x_0,\ldots,x_{d-1})=\rho(x;\sum_{i=0}^{d-1} \phi(x;x_i)),
\end{equation}
for suitable transformations $\phi$ and $\rho$. By specifying $\rho,\phi$ as neural networks, we get the basic formula of vanilla message-passing.

\subsubsection{Shift-invariant aggregation.}
Vanilla message-passing relies on permutation-invariant aggregation and it does not take into account the neighbor's order. Thus we are going to discuss the weaker case of aggregation function $f$ and assume that $f$ is shift-invariant, i.e.
$$
f(x;x_{0},\ldots,x_{d-1})=f(x;x_{0+p},\ldots,x_{d-1+p}),
$$
for any shift by a number $p\in{0,\ldots,d-1}$, where the additions on indices are performed modulo $d$. This assumption is consistent with our initial requirement that shifted orders are equivalent. 

The following theorem gives a general formula for shift-invariant mappings. 
\begin{theorem}
\label{theorem:1}
The function $f$ is shift-invariant if and only if $f$ can be written as:
\begin{equation*} \label{eq:o}
f(x_0,\ldots,x_{d-1}) = \sum_{p=0}^{d-1} g(x_{0+p},\ldots,x_{d-1+p})
\end{equation*}
for an arbitrary function $g$. 
\end{theorem}

\begin{proof}
If $f$ is shift invariant, then $f(x_0,\ldots,x_{d-1})=f(x_{0+p},\ldots,x_{d-1+p})$ for every $p$, and consequently
$$
f(x_0,\ldots,x_{d-1}) = \sum_{p=0}^{d-1} \frac{1}{d}f(x_{0+p},\ldots,x_{d-1+p}).
$$
On the other hand, if the function $f$ can be written as $ \sum_{p=0}^{d-1} g(x_{0+p},\ldots,x_{d-1+p})$, then it is shift-invariant for arbitrary function $g$.
\end{proof}

Following the above theorem, we get a general formula for shift-invariant aggregation applicable to graphs:
\begin{equation}\label{eq:shiftinv}
x'= \rho(x;\sum_{p=0}^{d-1} \psi(x;x_{{0+p}},\ldots,x_{d-1+p})),
\end{equation}
for suitable $\rho$ and $\psi$, where all additions are performed modulo $d$. 

Now, we want to ensure that our function $f$ is not only shift-invariant but also order-sensitive

\subsubsection{Order-sensitive message-passing.}

Let us assume that we are in the class of shift-invariant transformations. We are going to specify the formula \eqref{eq:shiftinv} to obtain ab aggregation, which is sensitive to any permutation other than shift. More precisely, we say that $f$ is order-sensitive if and only if for every permutation $\pi$, we have:
$$
f(x;x_0,\ldots,x_{d-1}) = f(x;x_{\pi(0)},\ldots,x_{\pi(d-1)}) \iff \pi(i) = (i + k) \Mod d.
$$

Let us investigate typical functions $\psi$ in formula \eqref{eq:shiftinv}, which can be implemented using neural networks. We start with the simplest case, where $\psi$ is linear. Then 
$$
\sum_{p=0}^{d-1}\psi(x_{0+p},\ldots,x_{d-1+p}) 
= \sum_{p=0}^{d-1}\sum_{i=0}^{d-1} w_i x_{i+p}
= \sum_{i=0}^{d-1}w_i\sum_{p=0}^{d-1} x_{i+p}
= \sum_{i=0}^{d-1}w_i\sum_{j=0}^{d-1} x_{j}, 
$$
does not depend on the order of the neighbors $(x_0, ..., x_{d-1})$. To construct more complex functions, we use an arbitrary Multi-Layer Perceptron (MLP) as $\psi$. Since MLPs are universal approximators (for a sufficiently large number of hidden units), we can find such parameters $\theta$ that $\sum_{p=0}^{d-1}\psi_{\theta}(x_{\pi(0)+p},\ldots,x_{\pi(d-1)+p}) $ returns a different value for every permutation $\pi$ that is not a shift. Therefore our aggregation scheme with $\psi$ given by MLP can learn order-sensitive mapping. 

Following the above observations, we implement our order-sensitive message-passing using MLP as $\psi$. To match our construction to various numbers of neighbors in a graph, we restrict $\psi$ to be $k$-ary  (denoted as $\psi^k$) for some fixed $k > 1$ and overload it so that:
$$
\psi^k(x_0,\ldots,x_{d-1})=
\psi^k(x_0,\ldots,x_{d-1},\underbrace{0,\ldots,0)}_{k-d}
\text{ for }d < k.
$$.

Given that, we implement the \cref{eq:shiftinv} with the following neural network layer:
\begin{equation}
\label{eq:osmp}
    \begin{gathered}
    x' = Wx + \sum_{p=0}^{d-1}\psi^k(x_{0+p},...,x_{k-1+p}), \\
    \psi^k(x_{0+p},...,x_{k-1+p}) = W_1\sigma (W_2(x_{0+p} | ... | x_{k-1+p})).
    \end{gathered}
\end{equation}
Our $k$-ary message function $\psi^k$ is composed of concatenation operator $|$ and two-layer MLP with ELU as $\sigma$. Intuitively, the output of
$\psi^k(x_{0+p}, ..., x_{k-1+p})$ can be seen as a message obtained jointly from $k$ consecutive neighbors starting from a neighbor $p$ in order $(x_0, ..., x_{d-1})$ which is illustrated in \cref{fig:fk}. 

\begin{figure}[t]
  \includegraphics[width=0.99\linewidth]{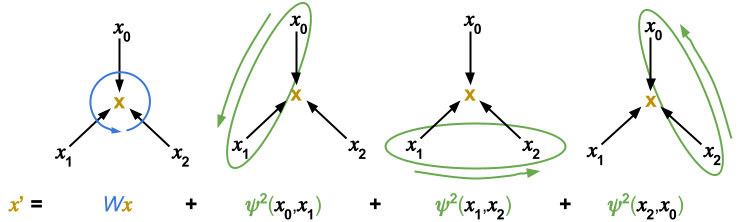}
  \centering
  \caption{An illustration of our update rule for node $x$ with 3 ordered neighbors $(x_0, x_1, x_2)$ and for $k=2$. We see that $\psi^k$ is used to embed pairs of consecutive nodes.}
  \label{fig:fk}
\end{figure}

\section{ChiENN: Chiral-Aware Neural Network}
\label{sec:chienn}

In this section, we apply the order-sensitive message-passing framework to molecular graphs. We show that order-sensitive aggregation is a key factor for embracing molecular chirality. Roughly speaking, in contrast to vanilla message-passing, the proposed \our{} (Chiral-aware Edge Neural Network) is able to distinguish enantiomers, where one molecule is a mirror image of the second. Although we evaluate the ChiENN model in the context of molecular property prediction, the proposed model can be applied to any 3D graph task with the notion of chirality.

To construct \our{} based on our order-sensitive message-passing scheme from \cref{eq:osmp}, we need to define a notion of neighbors' order in molecular graphs that grasps the concept of chirality (see \cref{fig:neighbors_ordering_intuition}). We introduce this notion of order for edge (dual) molecular graphs and provide a simple transformation from standard molecular graphs to edge molecular graphs. Therefore the rest of the section is organized into three subsections:
\begin{enumerate}
    \item \textbf{Edge Graph} describing the transformation from a molecular graph to its edge (dual) form used in our ChiENN model,
    \item \textbf{Neighbors Order} defining the order of the neighbors in an edge graph,
    \item \textbf{Chiral-Aware Update} constructing order-sensitive update rule using our order-sensitive framework from the \cref{sec:osmp-scheme}.
\end{enumerate}

\begin{figure}
  \includegraphics[width=0.7\linewidth]{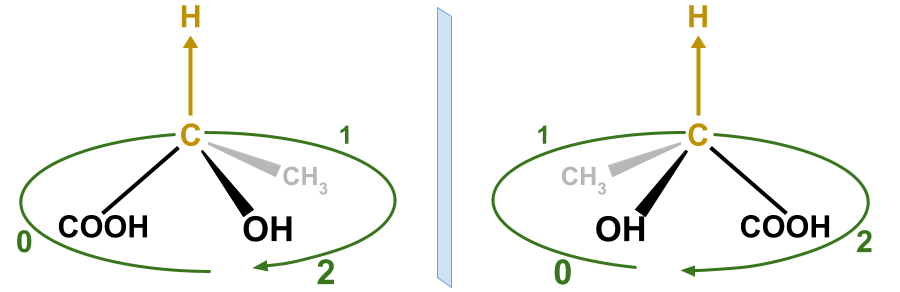}\centering
  \caption{An intuitive illustration of the neighbor ordering in a molecule. First, we pick a directed bond from atom C to H and then order the rest of the neighbors around that bond. We observe that for a chiral molecule (left) and its mirror image (right), we obtain different orders of the $COOH$, $CH_3$, and $OH$ constituents.}
  \label{fig:neighbors_ordering_intuition}
\end{figure}

\subsection{Edge Graph}
\label{sec:chienn-edge-graph}
Let us suppose, we have a directed graph $G=(X, C, E)$ that represents a concrete conformation (3D embedding) of a molecule. The node encoding $x_i \in X$ corresponds to an $i$-th atom from a molecule, $c_i \in C \subseteq \R^3$ are its coordinates in 3D space, and the edge encoding $e_{ij}\in E$ represents a bond between $i$-th and $j$-th atoms. 

To make the definition of neighbor order straightforward, our ChiENN model operates on an edge (dual) graph $G'=(X', C', E')$ which swaps nodes with edges from the original graph $G$. It means that the node $x_{ij} \in X'$ represents the edge $e_{ij}\in E$, while the edge $e_{ij, jk} \in E'$ represents the node $x_j$ that connects edge $e_{ij}\in E$ with $e_{jk} \in E$. Similarly, $c'_{ij} \in \R^3 \times \R^3$ is now a 3D coordinate vector that links positions $c_i$ and $c_j$.  Formally, we have:
\begin{align*}
\label{eq:edge-graph}
    X' & = \{x_{ij}=e_{ij} : e_{ij} \in E \}, \\
    C' & = \{c_{ij}=c_i | c_j : c_i, c_j \in C, e_{ij} \in E \}, \\
    E' & =\{e_{ij, jk}=e_{ij} | x_j | e_{jk}: e_{ij}, e_{jk} \in E, x_j \in X \},
\end{align*}
where $|$ stands for a concatenation operator. Clearly, the constructed edge graph $G'=(X', C', E')$ can be fed to any GNN that can take as an input the original graph $G=(X, C, E)$.

\subsection{Neighbors Order}
\label{sec:chienn-order}
In an edge molecular graph $G=(X, C, E)$, a node $x_{jk} \in E$ represents a directed bond from atom $j$ to atom $k$ in the original molecule. It is assigned with a 3D vector $c_{jk} \in C \subseteq \R^3 \times \R^3$ spanned from atom $j$ to atom $k$. Therefore, we will sometimes refer to nodes as if they were 3D vectors.

Let us consider the node $x_{jk}$ and the set of its incoming neighbors:  $N(x_{jk})=\{x_{i_1j}, x_{i_2j}, ..., x_{i_dj}\}$. By construction of $G$, every node $x_{jk}$ has a corresponding parallel node $x_{kj}$. For simplicity, we will treat this parallel node separately and exclude it from the set of neighbors, i.e. $x_{kj} \notin N(x_{jk})$.

The construction of the neighbors $N(x_{jk})$ order is illustrated in \cref{fig:neighbors_ordering} and consists of two steps:
\begin{enumerate}
    \item Transformation: first, we perform a sequence of 3D transformations on $x_{jk}$ and $N(x_{jk})$ to make $x_{jk}$ anchored to coordinate origin, perpendicular to $yz$ plane and pointed away from the observer (see \cref{fig:neighbors_ordering} b)).
    \item Sorting: second, we project the transformed neighbors $N(x_{jk})$ to the $yz$ plane and sort the projections by the angle to the $y$ axis.
\end{enumerate}
Details of the above construction are presented in the supplementary materials.

\begin{figure}[t]
  \includegraphics[width=0.87\linewidth]{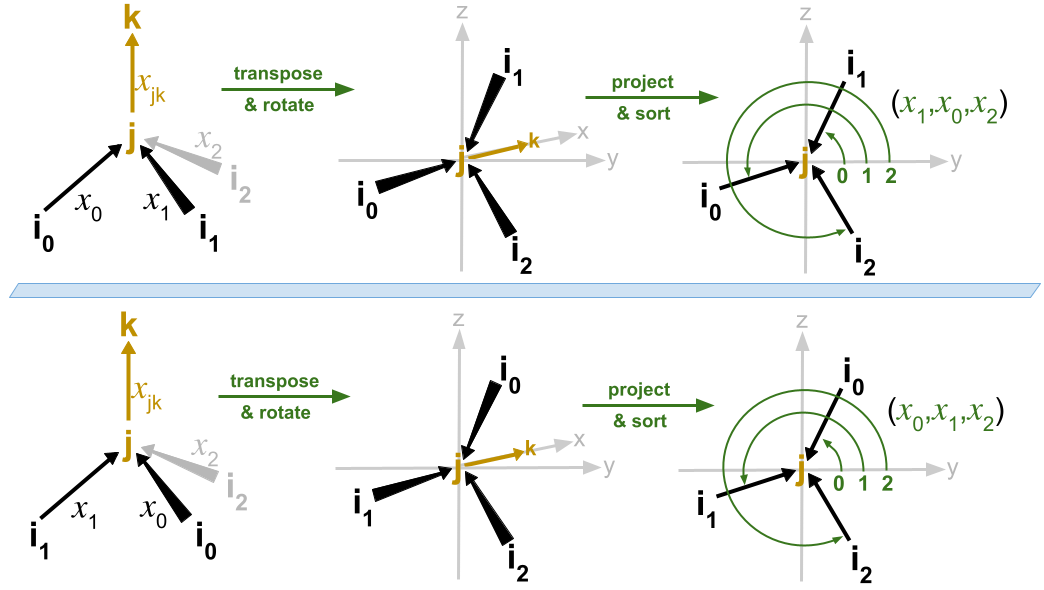}\centering
  \caption{An illustration of ordering the neighbors $\{x_0, x_1, x_2\}$ of $x_{jk}$ for a chiral molecule (top) and its mirror image (bottom) around the chiral center $j$. First, we perform a sequence of 3D transformations on $x_{jk}$ and its neighbors to make $x_{jk}$ anchored to coordinate origin, perpendicular to $yz$ plane and pointed away from the observer. Next, we project the transformed neighbors to the $yz$ plane and sort the projections by the angle to the $y$ axis. We see that for the chiral molecule (top) and its mirror image (bottom), we obtained non-equivalent orders $(x_1, x_0, x_2)$ and $(x_0, x_1, x_2)$.}
  \label{fig:neighbors_ordering}
\end{figure}

Two observations can be made regarding the above construction:
\begin{observation}
The above construction returns non-equivalent orders for a chiral center and its mirror image.
\end{observation}
\begin{observation}
Any $SE(3)$ transformation of a molecule coordinates $C$ and any internal rotation of its bonds (conformation) can only change the shift of the order $o$, resulting in equivalent order $o'$.
\end{observation}

Therefore, the above construction grasps the notion of chirality in a molecule and is additionally $SE(3)$- and conformation-invariant.

We artificially excluded $x_{kj}$ from a set of $x_{jk}$ neighbors, because its parallel to $x_{jk}$ and therefore its angle to $y$ axis after the sequence of transformations is undefined. In theory, another neighbor $x_{ij}$ can also be parallel to $x_{jk}$ and should also be excluded from the neighbor set, but we have not observed such a case in our experiments and decided not to take it into account.

\subsection{Chiral-Aware Update}
\label{sec:chienn-update}
Once we transformed a molecular graph to edge (dual) molecular graph $G=(X, E, C)$ using transformation from \cref{sec:chienn-edge-graph} and assigned every node $x_{jk}$ with an order of its neighbors $(x_1, ..., x_{d})$ using construction from \cref{sec:chienn-order}, we can define the order-sensitive update rule of our ChiENN model:
\begin{equation}
\label{eq:chienn}
    \begin{gathered}
    x_{jk}' = W_1x_{jk} + W_2x_{kj} + \sum_{p=0}^{d-1}\psi^k(x_{0+p},...,x_{k-1+p}), \\
    \psi^k(x_{0+p},...,x_{k-1+p}) = W_3\sigma (W_4(x_{0+p} | ... | x_{k-1+p})),
    \end{gathered}
\end{equation}
where $\psi^k$ is $k$-ary message function and $\sigma$ is ELU non-linear activation. The update rule is almost the same as that from \cref{eq:osmp}, but here we add a term that explicitly embeds $x_{kj}$ node, which was artificially excluded from the order of the $x_{jk}$ neighbors.

\section{Experiments}
\label{sec:experiments}

We compare ChiENN with several state-of-the-art models on a variety of chiral-sensitive tasks. Details of experiments are described in Section~\ref{sec:setup}, while the results can be found in Section~\ref{sec:benchmark}. Furthermore, to validate design choices behind ChiENN we also conducted an ablation study, presented in Section~\ref{sec:ablations}.

\subsection{Set-up}
\label{sec:setup}

\subsubsection{Datasets.} We conduct our experiments on five different datasets affected by molecule chirality. First, two datasets proposed in \cite{chiro} which are designed specifically to evaluate the capability of a model to express chirality: classification of tetrahedral chiral centers as R/S (which should be a necessary, but not sufficient, condition to learn meaningful representations of chiral molecules); and enantiomer ranking, in which pairs of enantiomers with enantioselective docking scores were selected, and the task was to predict which molecule of the pair had a lower binding affinity in a chiral protein pocket. 

Second, the binding affinity dataset, which is an extension of the previously described enantiomer ranking, with the same underlying molecules, but the task being regression of the binding affinity.

Additionally, we take two datasets from the MoleculeNet benchmark \cite{moleculenet} that do not explicitly require prediction of molecule chirality, but contain some percentage of molecules with chiral centers, and the underlying biological task in principle might be chirality-dependant: BACE, a binary classification dataset for prediction of binding results for a set of inhibitors of human $\beta$-secretase 1 (BACE-1) \cite{bace}; and Tox21, a multilabel classification dataset containing qualitative toxicity measurements on 12 different targets, including nuclear receptors and stress response pathways.

\begin{table*}[t]
\centering
\caption{Comparison of ChiENN-based approaches with the reference methods on chiral-sensitive tasks. Methods are split into groups by the underlying base model, except for the bottom group which includes models specifically designed to be chiral-sensitive. We \textbf{bold} the best results in every group and \underline{underline} the best results across all groups. All variations of our method (ChiENN, SAN+ChiENN, and GPS+ChiENN) significantly outperform current state-of-the-art chiral-sensitive models. Note that for the R/S task, we omitted the results for models with chiral tags encoded in node features, for which the task is trivial.}
\scriptsize
\label{tab:benchmark-chiral}
    \begin{tabularx}{0.75\textwidth}{lYYY}
    \toprule
    \multirow{3}{*}{Model} & \multirow{2}{*}{R/S} & Enantiomer & Binding \\
    & & ranking & affinity \\
    \cmidrule(lr){2-4}
    & Accuracy $\uparrow$ & R. Accuracy $\uparrow$ & MAE $\downarrow$ \\
    \midrule
DMPNN       & 0.500±0.000                      & 0.000±0.000                       & 0.310±0.001            \\
DMPNN+tags  & - & \textbf{0.701±0.003}          & \textbf{0.285±0.001}              \\
    \midrule
GPS         & 0.500±0.000                       & 0.000±0.000                      & 0.330±0.003               \\
GPS+tags    & - & 0.669±0.037          & 0.318±0.004                 \\
GPS+ChiENN  & \underline{\textbf{0.989±0.000}}          & \textbf{0.753±0.004}          & \textbf{0.258±0.001}                \\
    \midrule
SAN         & 0.500±0.000                       & 0.000±0.000                      & 0.317±0.004         \\
SAN+tags    & - & 0.722±0.004          & 0.278±0.003                 \\
SAN+ChiENN  & \textbf{0.987±0.001}                    & \underline{\textbf{0.764±0.005}}            & \underline{\textbf{0.257±0.002}}        \\
    \midrule
ChIRo       & 0.968±0.019          & 0.691±0.006          & 0.359±0.009       \\
Tetra-DMPNN & 0.935 ±0.001                   & 0.690±0.006             & 0.324±0.026                   \\
ChiENN      & \underline{\textbf{0.989±0.000}}          & \textbf{0.760±0.002} & \textbf{0.275±0.003}             \\
    \bottomrule
    \end{tabularx}
\end{table*}

\subsubsection{Reference methods.} As reference models we consider several state-of-the-art neural network architectures for processing graphs, both chirality-aware and general: GPS \cite{gps}, SAN \cite{san}, DMPNN \cite{dmpnn}, ChIRo \cite{chiro}, and Tetra-DMPNN \cite{tetra-dmpnn}. For models not designed to process chirality, that is DMPNN, GPS, and SAN, we additionally considered their variants with chiral atom tags included in the node features, similar to \cite{chiro}. For the proposed approach we consider both a pure model obtained by stacking several ChiENN layers, as well as combining ChiENN layers with other architectures (ChiENN+GPS and ChiENN+SAN).

\subsubsection{Training details.} All models were trained using Adam optimizer for up to 100 epochs, with a cosine learning rate scheduler with 10 warm-up epochs and gradient norm clipping, following the set-up of \cite{gps}. Cross-entropy and L1 loss functions were used for classification and regression, respectively. Note that in contrast to \cite{chiro}, to keep the set-up consistent across models we did not use triplet margin loss for ChIRo, and observed worse results than reported in \cite{chiro}. 

We also performed a grid search with the identical budget (see \cref{apx:exp-grid}). For all datasets and models, we reported results averaged from three runs.

For enantiomer ranking, binding affinity, and R/S, we used data splits provided by \cite{chiro} and for BACE and Tox21, we used random splits with a train-valid-test ratio of 7:1:2. For each model and dataset, we report mean results from 3 independent runs with the best parameters picked by grid search.

\subsubsection{Evaluation.} Note that for the binding rank task we used accuracy modified with respect to \cite{chiro}. We required the difference between the predicted affinity of two enantiomers to be higher than the threshold of 0.001. This led to ranking accuracy being equal to 0 for models unable to distinguish chiral molecules.

\subsection{Comparison with Reference Methods}
\label{sec:benchmark}

In this section, we compare ChiENN-based networks with state-of-the-art reference architectures using the experimental setting described in \cref{sec:setup}. 

\subsubsection{Chiral-sensitive tasks.} The results on chiral-sensitive tasks are presented in Table~\ref{tab:benchmark-chiral}. For both the enantiomer ranking and binding affinity, ChiENN-based approaches achieved the best results, producing a significant improvement in performance over the state-of-the-art chiral-aware architectures, that is ChIRo and Tetra-DMPNN. For both GPS and SAN, there was a significant improvement in performance due to the addition of ChiENN layers when compared to chiral tag inclusion. It demonstates that ChiENN model can enable chiral-awareness  demonstrating the general usefulness of the proposed layer, and the fact that it can be combined with a model preferred in a given task.

Finally, as expected, all of the chirality-aware methods can properly distinguish chiral centers in the R/S task, while the baselines that do not capture the concept of chirality (DMPNN, GPS and SAN) cannot. Note that for this task, we omitted the results for models with chiral tags encoded in node features, for which the task is trivial.

\subsubsection{Remaining tasks.} The results on BACE and Tox21 tasks are in Table~\ref{tab:benchmark-rest}. We see that the ChiENN model achieves results comparable to state-of-the-art models, however the influence of chirality-sensitiveness on these tasks is not clear. For SAN we actually observed a slight drop in performance when using ChiENN layers, and for GPS the results remained roughly the same. The possible explanations for that might be either 1) lack of importance of chirality on predicted tasks, or 2) small dataset size, leading to overfitting in presence of chiral information. Our conclusion is that ChiENN layers significantly improve the performance in chiral-sensitive tasks, and produce comparable results in the other tasks, where the influence of chirality is not clear. We believe that further investigation on the influence of chirality on the tasks commonly used in the molecular property prediction domain would be beneficial and we leave it for future work.  


\begin{table*}[h]
\centering
\caption{Comparison of ChiENN-based approaches with the reference methods on tasks \textbf{not} explicitly requiring chirality. 
We see that the ChiENN model achieves results comparable to state-of-the-art models.}
\scriptsize
\label{tab:benchmark-rest}
    \begin{tabularx}{0.6\textwidth}{lYY}
    \toprule
    \multirow{2}{*}{Model} & BACE & Tox21 \\
    \cmidrule(lr){2-3}
    & AUC $\uparrow$ & AUC $\uparrow$ \\
    \midrule
DMPNN               & \textbf{0.847±0.015} & 0.813±0.008          \\
DMPNN+tags          & 0.840±0.004          & \textbf{0.824±0.006}          \\
    \midrule
GPS       & \textbf{0.841±0.004}          & 0.821±0.000          \\
GPS+tags            & 0.812±0.017          & \textbf{0.825±0.002}          \\
GPS+ChiENN          & 0.839±0.008          & 0.821±0.007          \\
    \midrule
SAN                  & \underline{\textbf{0.846±0.012}}          & \textbf{0.842±0.007}          \\
SAN+tags              & 0.829±0.009          & 0.841±0.004          \\
SAN+ChiENN  &  0.826±0.014          & 0.834±0.005          \\
    \midrule
    ChIRo       &  0.815±0.010          & \underline{\textbf{0.847±0.005}} \\
Tetra-DMPNN   & 0.824±0.017          & 0.807±0.003          \\
ChiENN               & \textbf{0.838±0.003}          & 0.838±0.003    \\
    \bottomrule
    \end{tabularx}
\end{table*}

\subsection{Ablation Studies}
\label{sec:ablations}

\subsubsection{Comparison of $k$-ariness of the message function.} We began with an analysis of the impact of $k$-ariness (Equation~\ref{eq:osmp}) of the message function used by ChiENN. Specifically, in this experiment, we used the pure variant of ChiENN, which is a graph neural network using ChiENN layers as message-passing layers. We varied $k \in \{1, 2, 3\}$, where $k = 1$ disables the ability of the network to distinguish enantiomers as it collapses our order-sensitive message passing scheme from \cref{eq:osmp} to vanilla message-passing from \cref{eq:vanilla-message-passing}. We considered values of $k$ up to 3 since it corresponds to the airiness of standard chiral centers observed in the edge graphs (see \cref{sec:chienn-edge-graph}) of molecules.

The results are presented in Table~\ref{tab:ablation-k}. As expected, choosing $k = 1$ leads to a failure in distinguishing enantiomers (makes message passing permutation invariant), as demonstrated by minimum performance in R/S and enantiomer ranking tasks. Interestingly, for most datasets choosing $k = 2$ was sufficient, leading to a comparable performance to $k = 3$. The only exception to that was BACE dataset, for which a noticeable drop in performance was observed when using $k = 2$. We used $k = 3$ in the remainder of this paper.

\begin{table*}[t]
\caption{Comparison of $k$-ariness of the message function.}
\scriptsize
\label{tab:ablation-k}
    \begin{tabularx}{\textwidth}{cYYYYY}
    \toprule
    \multirow{3}{*}{$k$-ary} & \multirow{2}{*}{R/S} & Enantiomer & Binding & \multirow{2}{*}{BACE} & \multirow{2}{*}{Tox21} \\
    & & ranking & affinity & & \\
    \cmidrule(lr){2-6}
    & Accuracy $\uparrow$ & R. Accuracy $\uparrow$ & MAE $\downarrow$ & AUC $\uparrow$ & AUC $\uparrow$ \\
    \midrule
    $k = 1$ & 0.500±0.000 & 0.000±0.000 & 0.328±0.000 & 0.831±0.028 & 0.833±0.005 \\
    $k = 2$ & \textbf{0.989±0.001} & 0.759±0.003 & \textbf{0.267±0.001} & 0.788±0.014 & 0.836±0.004 \\
    $k = 3$ & \textbf{0.989±0.000} & \textbf{0.760±0.002} & 0.275±0.003 & \textbf{0.838±0.003} & \textbf{0.838±0.003} \\
    \bottomrule
    \end{tabularx}
\end{table*}

\subsubsection{Using ChiENN layer with existing models.} Secondly, we conducted an ablation of different design choices that can be made to enable enantiomer recognition within the existing architectures. Specifically, we focused on the GPS model and considered using three different strategies: conversion to edge graph proposed in this paper, the inclusion of chiral tags in the node features of the graph, and finally, replacement of message passing layers with ChiENN layers.

The results are presented in Table~\ref{tab:ablation-layer}. Several observations can be made: first of all, in the case of R/S task, we can see that both using chiral tags and the ChiENN layers allows us to properly recognize chiral centers (and as stated before, due to the simplicity of the task, good performance here is a necessary, but not sufficient, requirement for learning meaningful chiral representations). 

Secondly, using ChiENN layers significantly improves the performance in the enantiomer ranking (explicitly requiring chirality) and binding affinity (implicitly requiring it) tasks, more than simply including chiral tags. Interestingly, combining chiral tags with edge graph transformation improves the performance compared to using the tags alone (though not as much as using ChiENN layers), suggesting that it might be a feasible general strategy. 

Finally, the results on two remaining tasks, that is BACE and Tox21, for which the impact of chirality is unclear, are less straightforward: in the case of BACE, GPS with edge graph transformation achieves the best performance, and in the case of Tox21, using both the edge graph transformation and including the chiral tags. However, we can conclude that using ChiENN layers outperforms simply including chiral tags in tasks requiring chirality, and have comparable performance to the baseline GPS in other tasks.

\begin{table*}[t]
\caption{Ablation study of different design choices for GPS+ChiENN model. The "Graph" column indicates conversion to the edge graph; the "Tags" column indicates the inclusion of the chiral tags as node features; the "ChiENN" column indicates the usage of ChiENN as a message-passing layer.}
\scriptsize
\label{tab:ablation-layer}
    \begin{tabularx}{\textwidth}{cccYYYYY}
    \toprule
    \multirow{3}{*}{Graph} & \multirow{3}{*}{Tags} & \multirow{3}{*}{ChiENN} & \multirow{2}{*}{R/S} & Enantiomer & Binding & \multirow{2}{*}{BACE} & \multirow{2}{*}{Tox21} \\
    & & & & ranking & affinity & & \\
    \cmidrule(lr){4-8}
    & & & Accuracy $\uparrow$ & R. Accuracy $\uparrow$ & MAE $\downarrow$ & AUC $\uparrow$ & AUC $\uparrow$ \\
    \midrule
    No & No & No & 0.500±0.000 & 0.000±0.000 & 0.330±0.003 & 0.841±0.004 & 0.821±0.000 \\
    Yes & No & No & 0.500±0.000 & 0.000±0.000 & 0.306±0.001 & \textbf{0.851±0.007} & 0.821±0.007 \\
    No & Yes & No & \textbf{1.000±0.000} & 0.669±0.037 & 0.318±0.004 & 0.812±0.017 & 0.825±0.002 \\
    Yes & Yes & No & \textbf{1.000±0.000} & 0.720±0.002 & 0.283±0.008 & 0.802±0.019 & \textbf{0.838±0.003} \\
    Yes & No & Yes & 0.989±0.000 & \textbf{0.753±0.004} & \textbf{0.258±0.001} & 0.839±0.008 & 0.821±0.007 \\
    \bottomrule
    \end{tabularx}
\end{table*}


\section{Conclusions}
In this paper, we proposed and theoretically justify a general order-sensitive message-passing scheme that can be applied to any chiral-sensitive graph domain where chirality can be expressed by an order of the neighboring nodes. We used the proposed framework to construct a novel ChiENN layer that enables chirality awareness in any GNN model in the domain of molecular graphs, where chirality plays an important role as it can strongly alter the biochemical properties of molecules. Our experiments showed that the ChiENN layer allows to outperform the current state-of-the-art methods in chiral-sensitive molecular property prediction tasks.

\section*{Acknowledgements} The research of J. Tabor was supported by the Foundation for Polish Science co-financed by the European Union under the European Regional Development Fund in the POIR.04.04.00-00-14DE/18-00 project carried out within the Team-Net program. The research of P. Gaiński and M. Śmieja was supported by the National Science Centre (Poland), grant no. 2022/45/B/ST6/01117. For the purpose of Open Access, the author has applied a CC-BY public copyright license to any Author Accepted Manuscript (AAM) version arising from this submission.

\section*{Ethical Statement} As we consider our work to be fundamental research, there are no direct ethical risks or societal consequences; these have to be analyzed per concrete application.


%
%
%
\bibliographystyle{splncs04}
\bibliography{bib}

\appendix
\section{Neighbor Order Construction}
\label{app:label-order-construction}
In this section, we describe Transformation and Sorting steps of order construction from \cref{sec:chienn-order}.

\subsection{Transformation}
We perform a sequence of 3D transformations on $c_{jk}$ and $c_{i_1j}, ..., c_{i_dj} \in \mathds{R}^3 \times \mathds{R}^3$ to make $c_{jk}$ anchored to coordinate origin, perpendicular to $yz$ plane and pointed away from the observer. To simplify the notation, we will perform transformation on coordinates $c_i \in \mathds{R}^3$ that were used in the definition of the 6D coordinates $c_{ij}=c_i | c_j$ in \cref{sec:chienn-edge-graph}.

We define function to calculate angle between 2D vectors as $angle(a, b)=arccos \left(\frac{ab^T}{|a||b|}\right)$.

\begin{enumerate}
    \item We first transpose every point with $t(c) = c - c_j$, so that $c_{jk}$ is anchored to the coordinate origin, so $c_i=t(c_i)$,
    \item We calculate $\alpha_x=angle([(c_j)_y,(c_j)_z], [0, 1])$ and matrix $W_x$ representing the rotation along $x$ axis by $\alpha_x$ angle. We rotate the points, so $c_i=W_xc_i$,
    \item We calculate $\alpha_y=angle([(c_j)_x,(c_j)_z], [1, 0])$ and matrix $W_y$ representing the rotation along $y$ axis by $\alpha_x$ angle. We rotate the points, so $c_i=W_yc_i$.
\end{enumerate}


\subsection{Sorting}
After applying the transformation described in the previous versions, we can sort the 3D coordinates $c_{i_1},...,c_{i_d}$ by the angle $\alpha_{i}=angle([1, 0], [(c_i)_y, (c_i)_z])$ between the $y$-axis and their projections on $yz$ axis. Therefore, we obtain an order $(x_{i_{\pi(1)}k}, ..., x_{i_{\pi(d)}k})$ such that $\alpha_{i_{\pi(l)}} \leq \alpha_{i_{\pi(l+1)}}$.

\section{Experimental Details}
\label{apx:exp}
\subsection{Hyperparameter grids}
\label{apx:exp-grid}

We performed a grid search of parameters with identical budget: for all models, with learning rate $\in \{1\mathrm{e}{-3}, 1\mathrm{e}{-4}, 1\mathrm{e}{-5}\}$ and dropout $\in \{0, 0.2, 0.5\}$, and with model-dependent number of layers and layer dimensionality, chosen based on the parameters from corresponding papers: for DMPNN and Tetra-DMPNN, with layers $\in \{2, 4, 6\}$ and dimensionality $\in \{300, 600, 900\}$; for ChIRo, with layers $\in \{2, 3, 4\}$ and dimensionality $\in \{64, 128, 256\}$; and for GPS, SAN and ChiENN, with layers $\in \{3, 6, 10\}$ and dimensionality $\in \{64, 128, 256\}$. We restricted the grid search to a subset of a dataset of size at most 10000 molecules. As the computational costs of Tetra-DMPNN are high, for binding affinity, we took the optimal hyperparameters for DMPNN+tags as parameters for Tetra-DMPNN. We did similarly for SAN+ChiENN for R/S and binding rank and took the corresponding optimal hyperparameters from SAN.

\end{document}